%% file: main.tex
\pdfoutput=1
\documentclass{article}

\usepackage{microtype}
\usepackage{subfigure}
\usepackage{booktabs} 
\usepackage{hyperref}

\usepackage[accepted]{icml2019}

\input{includes.tex}

\icmltitlerunning{Learning Models from Data with Measurement Error: Tackling Underreporting}

\begin{document}

\twocolumn[
\icmltitle{Learning Models from Data with Measurement Error:\\Tackling Underreporting}



\icmlsetsymbol{equal}{*}

\begin{icmlauthorlist}
\icmlauthor{Roy Adams}{jhucs}
\icmlauthor{Yuelong Ji}{jhuph}
\icmlauthor{Xiaobin Wang}{jhuph}
\icmlauthor{Suchi Saria}{jhucs,jhust,bh}
\end{icmlauthorlist}

\icmlaffiliation{jhucs}{Department of Computer Science, Johns Hopkins University}
\icmlaffiliation{jhust}{Department of Applied Math and Statistics, Johns Hopkins University}
\icmlaffiliation{bh}{Bayesian Health}
\icmlaffiliation{jhuph}{Center on the Life Origins of Disease, Department of Population, Family, and Reporductive Health, Johns Hopkins University Bloomberg School of Public Health}

\icmlcorrespondingauthor{Roy Adams}{roy.james.adams@gmail.com}
\icmlsecondcorrespondingauthor{Suchi Saria}{ssaria@cs.jhu.edu}

\icmlkeywords{measurement error; label noise; weakly supervised learning}

\vskip 0.3in
]



\printArxivAffiliations{}
\input{abstract}

\section{Introduction}
\label{sec:introduction}
\input{intro}

\section{Background}
\label{sec:background}
\input{background}

\section{Model}
\label{sec:model}
\input{model}

\section{Identifiability}
\label{sec:identifiability}
\input{identifiability}
\section{Synthetic experiments}
\label{sec:synth}
\input{synth}

\section{Boston birth cohort}
\label{sec:bbc}
\input{bbc}

\section{Related work}
\label{sec:related}
\input{related}

\section{Discussion}
\label{sec:discussion}
\input{discussion}

\bibliography{references}
\bibliographystyle{icml2019}
\clearpage
\appendix
\input{appendix}


\end{document}

%% file: includes.tex
\usepackage{graphicx} 

\usepackage{booktabs} 
\usepackage{paralist} 
\usepackage{verbatim} 
\usepackage{amsmath}
\usepackage{amsfonts}
\usepackage{amsthm}
\usepackage{array}
\usepackage{graphicx}

\usepackage{bm}

\theoremstyle{plain}

\newtheorem{prop}{Proposition}

\newtheorem{theorem}{Theorem}
\newtheorem{corollary}{Corollary}

%% file: abstract.tex
\begin{abstract}
Measurement error in observational datasets can lead to systematic bias in inferences based on these datasets. As studies based on observational data are increasingly used to inform decisions with real-world impact, it is critical that we develop a robust set of techniques for analyzing and adjusting for these biases. In this paper we present a method for estimating the distribution of an outcome given a binary exposure that is subject to underreporting. Our method is based on a missing data view of the measurement error problem, where the true exposure is treated as a latent variable that is marginalized out of a joint model. We prove three different conditions under which the outcome distribution can still be identified from data containing only error-prone observations of the exposure. We demonstrate this method on synthetic data and analyze its sensitivity to near violations of the identifiability conditions. Finally, we use this method to estimate the effects of maternal smoking and opioid use during pregnancy on childhood obesity, two import problems from public health. Using the proposed method, we estimate these effects using only subject-reported drug use data and substantially refine the range of estimates generated by a sensitivity analysis-based approach. Further, the estimates produced by our method are consistent with existing literature on both the effects of maternal smoking and the rate at which subjects underreport smoking.
\end{abstract}

%% file: intro.tex
Measurement error in observational datasets can lead to systematic bias in inferences based these datasets. As studies using observational data are increasingly used to inform decisions with real-world impact, it is critical that we develop a robust set of techniques for analyzing and adjusting for these biases. Bias can mean different things in different contexts, but here we use it in the statistical sense to refer to inferences that are wrong in a non-random way \cite{casella2002statistical}.
Basing decisions on biased inferences can lead to real consequences and degrade trust in the use of observational data to inform decision making. For example, decisions about what resources should be allocated to preventing maternal drug use during pregnancy must be based on studies that rely on subject-reported drug use behaviors \cite{wang2002maternal}.
Analyzing and accounting for various potential sources of bias remains an open problem, but one with important real-world implications.

One reason that accounting for bias in observational data is a difficult problem is that it requires making further assumptions about what the sources and magnitudes of various biases might be. There is often substantial debate about the validity of these assumptions (e.g. \citet{knape2015estimates, solymos2016revisiting, knape2016assumptions}). As a result, it is frequently the case that observational studies relegate potential sources of bias to the discussion of limitations rather than performing quantitative bias analysis \cite{rothman2008modern}. It is our view that, to the extent possible, quantitative bias analysis should be presented alongside quantitative measures of uncertainty when presenting results from observational studies. It is therefore critical that we develop a robust set of tools for analyzing and accounting for observational bias. 

In this work we focus on bias caused by measurement error, the degree and direction of which depends heavily on the type of error \cite{carroll2006measurement,gustafson2003measurement}.
There are three common approaches for dealing with measurement error (more details on these approaches are provided in Section \ref{sec:background}). 
The first approach assumes that we have access to a data source containing error-free measurements (referred to as \emph{validation data}). This data can be used to estimate the distribution of errors in order to adjust appropriately when error-free measurements are not available. When validation data is available, this approach is clearly preferable, but such data is often difficult or impossible to gather. 

The second approach assumes that we can specify a small set of hypothetical error distributions which are then used to generate a corresponding range of inferences \cite{rothman2008modern}.
This set of hypothetical distributions can be specified using either domain knowledge or previous studies; however, if such knowledge is not available, then this type of sensitivity analysis may generate a wide range of inferences.

In the third approach, the error-free measurements are treated as unobserved variables, a model is specified that includes these variables, and inferences are made using only the observed error-prone data. If the modeling assumptions are correct, then this approach can give unbiased inferences without relying on validation data or specific knowledge about the amount of error present in the data.

In this work, we propose a method based on the third approach to account for a type of measurement error called \emph{exposure misclassification}. Exposure misclassification occurs when we are interested in estimating the distribution of an outcome $Y$ given a binary exposure $A$, but $A$ is subject to measurement error\footnote{Measurement error in discrete variables is referred to as \emph{misclassification}.}. For example, when estimating the effect of maternal drug use (exposure) on childhood obesity (outcome) using survey data, subjects have a tendency to \emph{underreport} (source of error) whether or not they have used drugs.

The primary contribution of this work is a method for estimating the outcome distribution when we are only able to observe a version of the exposure that is subject to underreporting. Underreporting is common in problems involving survey-based observations of sensitive behaviors and, to our knowledge, this is the first method that allows unbiased estimation directly from such data. We prove three different assumptions under which such estimation is possible (Section \ref{sec:identifiability}), enabling flexible application of our method to different problem settings.

Using synthetic data, we show that this approach substantially reduces estimation error compared with treating the error-prone observations as ground truth (Section \ref{sec:synth}). Finally, we use this approach to estimate the effects of maternal smoking and opioid use during pregnancy on childhood obesity using subject-reported drug use data. The effects of childhood obesity later in life can be severe and identifying potential causes of childhood obesity has significant public health implications.
Due to underreporting error, obtaining unbiased estimates of these effects directly from survey data has not been possible and, to our knowledge, the contributions made in this article enabled the first reported estimate of the effect of maternal opioid use on childhood obesity.
Estimates produced by the proposed method substantially refine the range of estimates generated by a sensitivity analysis based approach and are consistent with existing literature on the effects of maternal smoking and the rate at which subjects underreport smoking (Section \ref{sec:bbc}).

%% file: background.tex
Measurement error is encountered in a wide range of settings and there is extensive literature on various adjustment techniques. In this section, we contextualize our contributions with a review of the measurement error bias problem and some of the approaches to adjusting for this bias. We encourage the interested reader to seek out \citet{carroll2006measurement} or \citet{gustafson2003measurement} for full treatments of this topic.

\subsection{Measurement error bias}
Suppose that we are interested in estimating the distribution of response $Y$ given predictor $A$, denoted $p(Y|A)$. We may do this because we are interested in the parameters of this distribution or because we would like to predict $Y$ from future values of $A$. However, suppose that rather than observing $A$ in our training data, we observe $\tilde{A}$, a version of $A$ that is corrupted by measurement error. In cases where $p(Y|A) \neq p(Y|\tilde{A})$, ignoring the measurement error in $\tilde{A}$ may lead us to a biased estimate of $p(Y|A)$. 
The specific effect using $\tilde{A}$ has on our estimate will depend on how $\tilde{A}$ relates to $A$ and we will generally have to make some modeling assumptions about this relationship in order to adjust for it. For example, classical models for measurement error assume that $\tilde{A}$ equals $A$ plus a noise variable that is independent of $A$ \cite{carroll2006measurement}; however, these assumptions are not appropriate in many scenarios, including the one considered in this paper. 
Given a model for the error process, there are a number of ways we might go about adjusting for the bias caused by measurement errors.

\subsection{Adjusting for measurement error bias}
As discussed in the introduction, methods for adjusting for measurement error bias fall into three broad categories and choosing which method is appropriate for a particular problem will depend on what data and domain knowledge is available. 

\paragraph{Auxiliary data:} In the first category, it is assumed that some form of auxiliary data is available that allows us to estimate the error distribution directly. This auxiliary data may be validation data for which ground truth is available or it may be a second error-prone measurement of the same underlying variable. In the case of validation data, it is clear that the error distribution can be estimated; however, if the auxiliary data also contains errors, assumptions about the type and relationship of these errors are necessary to guarantee that the error distribution is estimable. For example, in Section \ref{sec:aux_meas} we prove for our model that the error distribution is estimable from two error-prone measurements which are independent given the true exposure. Once the error distribution is estimated, there are a number of ways one can correct for the missing ground truth observations such as imputation or sampling \cite{carroll2006measurement}. Using auxiliary data typically requires the least restrictive modeling assumptions and is therefore preferable when such data is available. It is unsurprising then that this approach forms the basis for much of the classic literature on measurement error \cite{carroll2006measurement,gustafson2003measurement} and is used in modern approaches when possible (e.g. \citet{pearl2012measurement}). In this work, however, we consider the case when no auxiliary data is available and these methods do not apply.

\paragraph{Sensitivity analysis:} When we cannot directly estimate the error distribution from data, we can instead consider a range of hypothetical error distributions. If this range is small, then we can simply enumerate them, correcting for the missing ground truth observations as above, to generate a corresponding range of plausible inferences. 
This approach does not rely on any auxiliary data and can be applied to the cases we consider in this paper; however, if we do not have sufficient domain knowledge to specify a small set of error distributions, the resulting range of inferences may be quite large (as we demonstrate in Section \ref{sec:bbc}).

\paragraph{Full likelihood:} A third and less common approach is to specify a joint model for the target $Y$, the unobserved ground-truth predictor $A$, and the error-prone measurement $\tilde{A}$ and then marginalize $A$ out of this joint model. This is often referred to as the \emph{full likelihood} approach \cite{carroll2006measurement,thomas1993exposure}. Even if correctly specified, the parameters of this joint model cannot, in general, be estimated from the observed data without further assumptions about the structure of the joint distribution; however, there has been some work to identify such cases. For example, \citet{rudemo1989random} show that a particular class of non-linear medication dose response models can be estimated even if the true dose is subject to measurement error. \citet{kuchenhoff1997segmented} show that a class of threshold regression models can be estimated from data with additive observation error in the predictors. The method presented in this paper is an example of the full likelihood approach applied the exposure misclassification problem.

%% file: model.tex
\begin{figure}[t!]
\vskip 0.2in
\begin{center}
\centerline{\includegraphics[width=0.5\columnwidth]{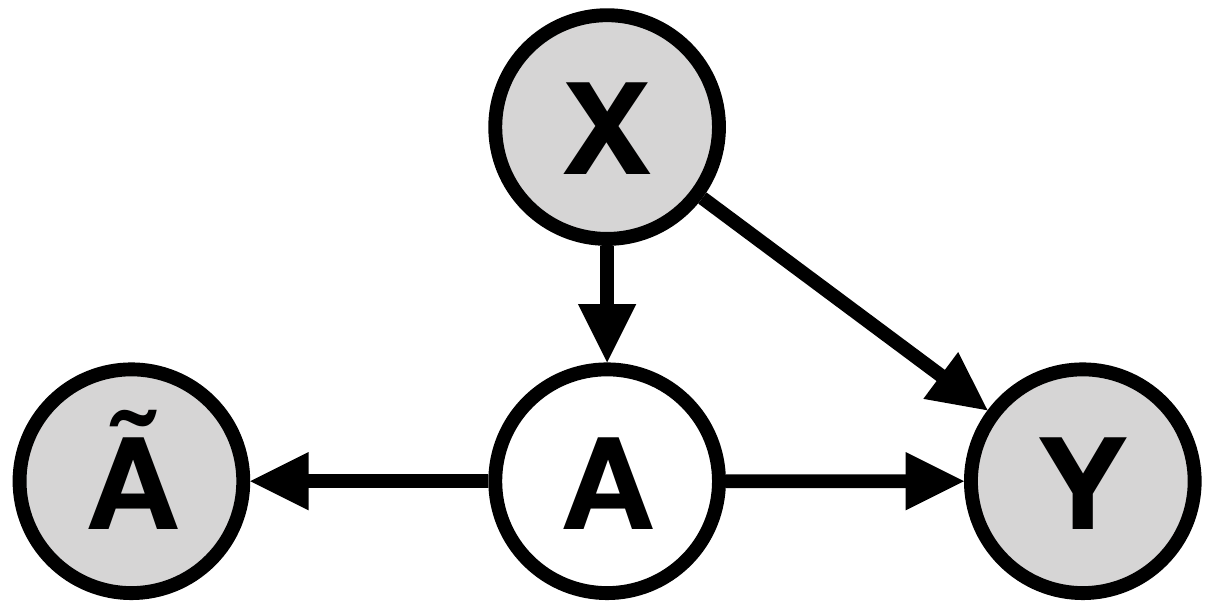}}
\caption{Proposed graphical model. $Y$ is the outcome of interest, $A$ is the true binary exposure, $\tilde{A}$ is the observed exposure, and $X$ is a vector of covariates. Grey variables are observed and white variables are not.}
\label{fig:graphical_model}
\end{center}
\vskip -0.2in
\end{figure}

Suppose that we are interested in estimating the conditional distribution of a binary outcome $Y\in\{0,1\}$ given a binary exposure binary exposure $A \in \{0,1\}$ and a set of known confounding variables $X\in\mathbb{R}^d$\footnote{Under the assumptions of consistency, positivity, and conditional exchangeability, this distribution could be used to estimate the causal effect of $A$ on $Y$ \cite{hernan2018causal}.}; however, rather than observing $A$, we observe an error-prone version $\tilde{A} \in \{0,1\}$. For example, $A$ might represent true drug use while $\tilde{A}$ represents subject-reported drug use. For reasonably high-dimensional covariates $X$, this is typically done using a parametric \emph{outcome model} $p_\theta(Y|A,X)$ which, if we were able to observe $A$, could be estimated directly from data. Instead, our goal is to estimate $p_\theta(Y|A,X)$ from a dataset $\mathcal{D}=\{(x_i,\tilde{a}_i,y_y)\}_{i=1}^N$ consisting only of the outcome, covariates, and error-prone exposure observation. 

In order to make this possible, we make two simplifying assumptions. First, we assume that the observed exposure $\tilde{A}$ is independent of the outcome and covariates given the true exposure or $\tilde{A} \perp X,Y | A$. In our drug use example, this means that the decision to misreport drug use is independent of known confounders such as age and income given the ground truth. When the conditional independence with $Y$ holds, the measurement error is referred to as \emph{non-differential}. This assumption is illustrated in the graphical model shown in Figure \ref{fig:graphical_model}. 

Second, we assume strict underreporting of the true exposure or $p(\tilde{A}=1|A=0) = 0$. For example, this assumption might encode our belief the non-users will not falsely report that they use drugs. Given these assumptions, we can write the conditional probability of $Y$ and $\tilde{A}$ given $X$ as\footnote{When it does not cause confusion, we simplify $p(\tilde{A}=\tilde{a}|A=a)$ to $p(\tilde{a}|a)$.}:
\begin{align*}
	p(y,\tilde{a}|x) = \sum_a p(\tilde{a}|a) p(a|x) p(y|a,x)
\end{align*}
where $\tau$ represents the underreporting rate. Finally, we will assume that the modeler has selected a parametric \emph{propensity model} $p_\phi(a|x) \in \mathcal{P}_\Phi$ and a parametric \emph{outcome model} $p_\theta(y|a,x) \in \mathcal{P}_\Theta$ for $p(a|x)$ and $p(y|a,x)$ respectively. We parameterize the \emph{error model} $p_\tau(\tilde{a}|a)$ as:
\begin{align*}
	p_{\tau}(\tilde{a}=0|a=1) &= 1 - p_{\tau}(\tilde{a}=1|a=1) = \tau\\
	p_{\tau}(\tilde{a}=0|a=0) &= 1 - p_{\tau}(\tilde{a}=0|a=0) = 1
\end{align*}
We can estimate these parameters jointly by maximizing the following log conditional likelihood:
\begin{align}
	\label{eq:lml}
	\mathcal{L}(\tau,\phi,\theta) = \sum_i \log \sum_a p_\tau(\tilde{a}_i|a) p_\phi(a|x_i) p_\theta(y_i|a,x_i)
\end{align}
Even in the case where $\phi$ and $\theta$ could be estimated from data containing the true exposure, we may not be able to estimate $\tau$, $\phi$, and $\theta$ from the available data without further assumptions. In the next section, we prove several such conditions under which the parameters of this model are estimable.
%
%
%
%
%

%% file: identifiability.tex
We prove three conditions under which the joint model $p_{\tau,\phi,\theta}(y,\tilde{a}|x)$ is identifiable\footnote{A parametric model $p_w\in\mathcal{P}_\mathcal{W}$ is \emph{identifiable} if $p_w=p_{w'} \implies w=w'$ for all $w,w'\in\mathcal{W}$. This is equivalent to saying that, in the limit of infinite data from $p_w$, we could identify the parameters that gave rise to this data. Importantly, this is a property of a model and is independent of the data.}. These results allow us to use the model from Section \ref{sec:model} for each of the three bias analysis and adjustment approaches described in Section \ref{sec:background}: sensitivity analysis, auxiliary data, and full likelihood.

\subsection{Known $\tau$}

The validation data and sensitivity analysis approaches are based on the idea that, when $p_\tau(\tilde{a}|a)$ is known, $p_\theta(y|a,x)$ is estimable from data containing measurement error. We prove that this is true for the model presented in Section \ref{sec:model} allowing us to use this model to perform sensitivity analysis by fixing $\tau$ at various values and estimating $\phi$ and $\theta$ from the observed data. Specifically, we have the following result:
\begin{prop}
	\label{prop:ident1}
	If $\tau$ is known and $p_{\phi,\theta}(y,a|x)$ is identifiable, then $p_{\tau,\phi,\theta}(y,\tilde{a}|x)$ is identifiable.
\end{prop}
\begin{proof}
By the identifiability of $p_{\phi,\theta}(y,a|x) = p_\theta(y|a,x)p_\phi(a|x)$, $(\phi,\theta) = (\phi',\theta')$ if and only if $p_{\phi,\theta}(y,a|x) = p_{\phi',\theta'}(y,a|x)$. Then for any $x$, $y$, and $\tau < 1$, marginalizing over $A$ to map from $p_{\phi,\theta}(y,a|x)$ to $p_{\phi,\theta}(y,\tilde{a}|x)$ is an invertible linear map, so
$p_\theta(y,\tilde{a}|x) = p_{\theta'}(y,\tilde{a}|x)$ if and only if $p_\theta(y,a|x) = p_{\theta'}(y,a|x)$.
\end{proof}
Unfortunately, it is frequently the case that we do not know $\tau$ and cannot specify a reasonably tight range of possible values, so we must estimate it from data.

\subsection{Multiple error-prone observations}
\label{sec:aux_meas}

Next, we consider the case where $\tau$ is unknown, but we have access to a second error-prone exposure observation. For example, we may have access to both subject-reported drug use and lab measurements of biomarkers for drug use. We assume the second observation is also subject to the conditional independence and underreporting assumptions described in Section \ref{sec:model}. In this scenario, both $\tilde{A}$ and $\tau$ are two dimensional vectors. In general, we cannot estimate $\tau$, $\phi$, and $\theta$ from the available data without further assumptions about $\tilde{A}$. One such assumption is that the two observations, $\tilde{A}_1$ and $\tilde{A}_2$, are conditionally independent given the true exposure $A$. When such data is available, we can use this result to estimate $\tau$, $\phi$, and $\theta$ with out any knowledge of the value of $\tau$.
\begin{prop}
	\label{prop:ident3}
	If $\tilde{A}_1 \perp \tilde{A}_2 | A$ and $p_{\phi,\theta}(y,a|x)$ is identifiable, then $p_{\tau,\phi,\theta}(y,\tilde{a}|x)$ is identifiable.
\end{prop}
\begin{proof}
	Assume for contradiction that the above condition holds and there exists $(\tau,\phi,\theta) \neq (\tau',\phi',\theta')$ such that $p_{\tau,\phi,\theta}(y,\tilde{a}|x) = p_{\tau',\phi',\theta'}(y,\tilde{a}|x)$ (i.e. $p_{\tau,\phi,\theta}(y,\tilde{a}|x)$ is not identifiable). Then the following equalities must hold:
\begin{align*}
	(1-\tau_1)(1-\tau_2)\pi_\phi(x) &= (1-\tau'_1)(1-\tau'_2)\pi_{\phi'}(x)\\
	\tau_1(1-\tau_2)\pi_\phi(x) &= \tau'_1(1-\tau'_2)\pi_{\phi'}(x)\\
	\tau_2(1-\tau_1) \pi_\phi(x) &= \tau'_2(1-\tau'_1)\pi_{\phi'}(x)
\end{align*}
where $\pi_\phi(x) = p_\phi(A=1|x)$. By applying some algebra to these equalities, we have that $\tau_1=\tau'_1$ and $\tau_2=\tau'_2$. Finally, by Proposition \ref{prop:ident1}, we have that $(\phi,\theta) = (\phi',\theta')$, which is a contradiction.
\end{proof}
%
Intuitively, we know there are no false positives, so the independence assumption $\tilde{A}_1 \perp \tilde{A}_2 | A$ allows us to estimate $\tau_1$ as the proportion of samples where $\tilde{A}_1 = 0$ among the samples where $\tilde{A}_2=1$ and vice versa.

\subsection{Single error-prone observation}
\label{sec:single}
Finally, we consider the most difficult case: when we only have access to a single error-prone observation and no knowledge of $\tau$. In this case, we can only guarantee identifiability of $p_{\tau,\phi,\theta}(y,\tilde{a}|x)$ by making further assumptions about the structure of the model. In Theorem \ref{thm:ident2}, we prove one such condition which constrains the structure of the propensity model $p_\phi(a|x)$. Further, we prove in Corollary \ref{cor:cor1} that three of the most common Bernoulli regression models meet this condition, allowing us to use these models to estimate $\tau$, $\phi$, and $\theta$ from data containing only a single error-prone observation.
\begin{theorem}
	\label{thm:ident2}
	If $p_{\phi,\theta}(y,a|x)$ is identifiable and for all $\alpha \in [0,1)$ and $\phi,\phi'\in\Phi$, there exists $x$ such that $p_\phi(A=1|x) \neq \alpha p_{\phi'}(A=1|x)$, then $p_{\tau,\phi,\theta}(y,\tilde{a}|x)$ is identifiable.
\end{theorem}
\begin{proof}
Assume for contradiction that the above condition holds and there exists $(\tau,\phi,\theta) \neq (\tau',\phi',\theta')$ such that $p_{\tau,\phi,\theta}(y,\tilde{a}|x) = p_{\tau',\phi',\theta'}(y,\tilde{a}|x)$ (i.e. $p_{\tau,\phi,\theta}(y,\tilde{a}|x)$ is not identifiable). By assumption, we have
\begin{align*}
	&(1-\tau)p_{\phi}(A=1|x)p_{\theta}(y|A=1,x)\\&=  (1-\tau')p_{\phi'}(A=1|x)p_{\theta'}(y|A=1,x)
\end{align*}
If $p_{\theta}(y|A=1,x) \neq p_{\theta'}(y|A=1,x)$, then $\frac{(1-\tau)p_{\phi}(A=1|x)}{(1-\tau')p_{\phi'}(A=1|x)} \neq 1$ and as a result $p_{\theta'}(y|A=1,x)$ does not sum to one. Therefore, $p_{\theta}(y|A=1,x) = p_{\theta'}(y|A=1,x)$ and
\begin{align*}
	(1-\tau)p_{\phi}(A=1|x) = (1-\tau')p_{\phi'}(A=1|x)
\end{align*}
Without loss of generality, assume that $\tau < \tau'$. Then, $p_{\phi}(A=1|x) =  \alpha p_{\phi}(A=1|x)$ where $\alpha = \frac{1-\tau'}{1-\tau} \in [0,1)$ which is a contradiction.
\end{proof}
This result is somewhat technical, but intuitively, it says that if there exists $\phi$, $\phi'$, and $\alpha\in[0,1)$ such that $p_\phi(A=1|x) = \alpha p_{\phi'}(A=1|x)$ then it will be impossible to distinguish $(\alpha\tau,\phi)$ from $(\tau,\phi')$ given the observed data. Fortunately, this condition holds for three common Bernoulli regression models.
\begin{corollary}
	\label{cor:cor1}
	If $A \not\perp X$, then a Bernoulli regression model $p_\phi(A=1|x) = \Psi^{-1}(\phi x)$ with a logit, probit, or complementary log-log (cloglog) link function $\Psi$ satisfies the identifiability condition in Theorem \ref{thm:ident2}.
\end{corollary}
\begin{proof}
	A Bernoulli regression model of the form $p_\phi(A=1|x) = \Psi^{-1}(\phi x)$ violates the condition in Theorem \ref{thm:ident2} if and only if there exists $\alpha \in [0,1)$, $\phi$, and $\phi'$ such that
\begin{align*}
	\Psi(\alpha \Psi^{-1}(\phi x)) = \phi' x
\end{align*}
For logistic regression, this is true if 
\begin{align*}
\log(\alpha) - \log(1-\alpha + e^{-\phi x}) = \phi' x
\end{align*}
For all $\alpha < 1$, the function $\log(1-\alpha + e^{-x})$ is non-linear in $x$, so this equality can only be true if $\phi x$ and $\phi' x$ are constants which is true only when $A \perp X$. The same argument can be applied to the probit and complementary log-log link functions (see the supplementary materials for details).
\end{proof}
These three results allow us to estimate the outcome model directly from the observed data in a variety of scenarios. Which result is most applicable will depend on the particular modeling problem. For example, if $\tau$ is known based on previous literature, then we should use Proposition \ref{prop:ident1} as this result makes the fewest assumptions about the structure of the model. Equipped with these results, we apply the full likelihood approach to estimation problems using both synthetic and real data.

%% file: synth.tex
\begin{figure*}[t!]
    \centering
    \subfigure[]{
		\includegraphics[width=0.31\textwidth]{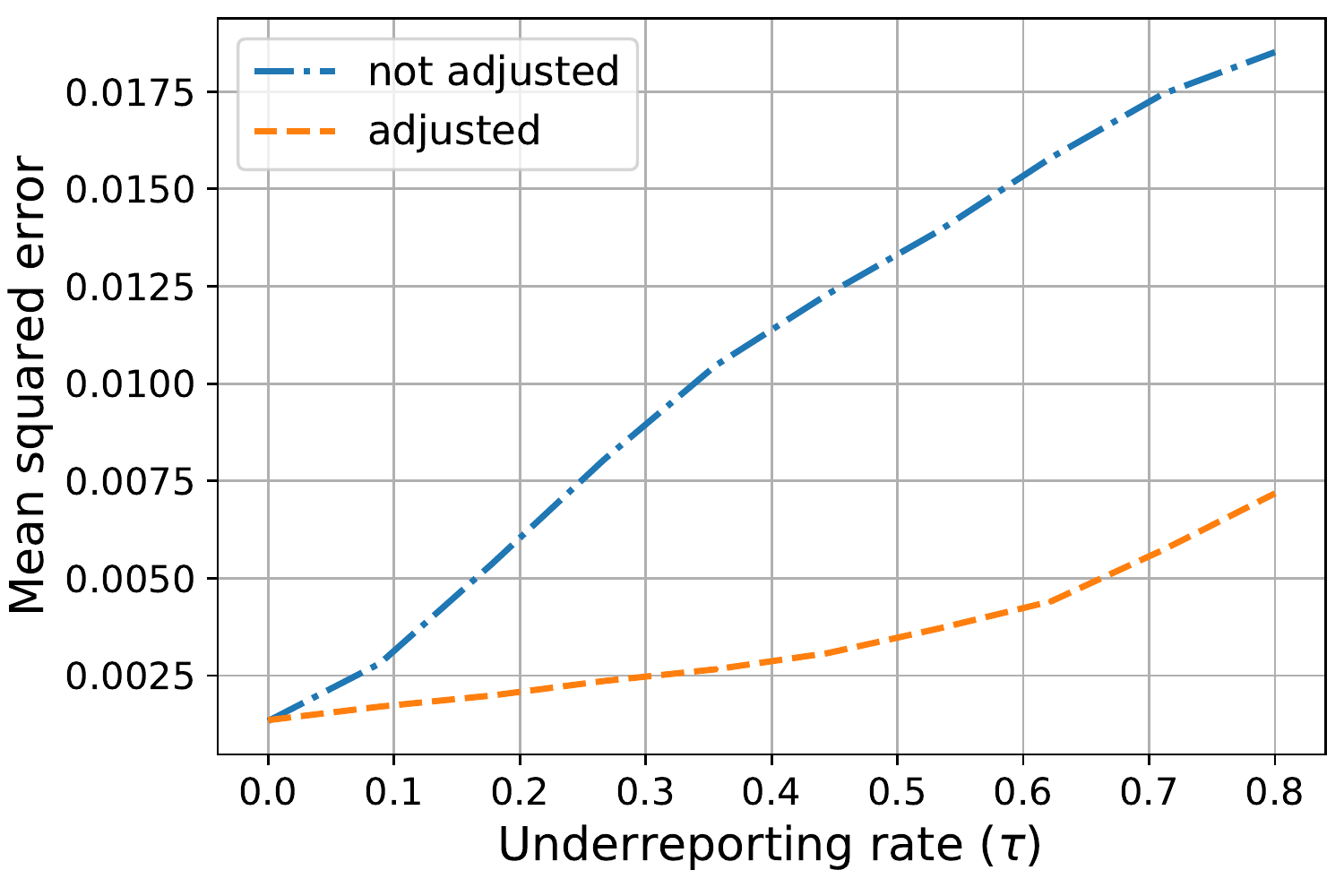}
	}
    \subfigure[]{
		\includegraphics[width=0.31\textwidth]{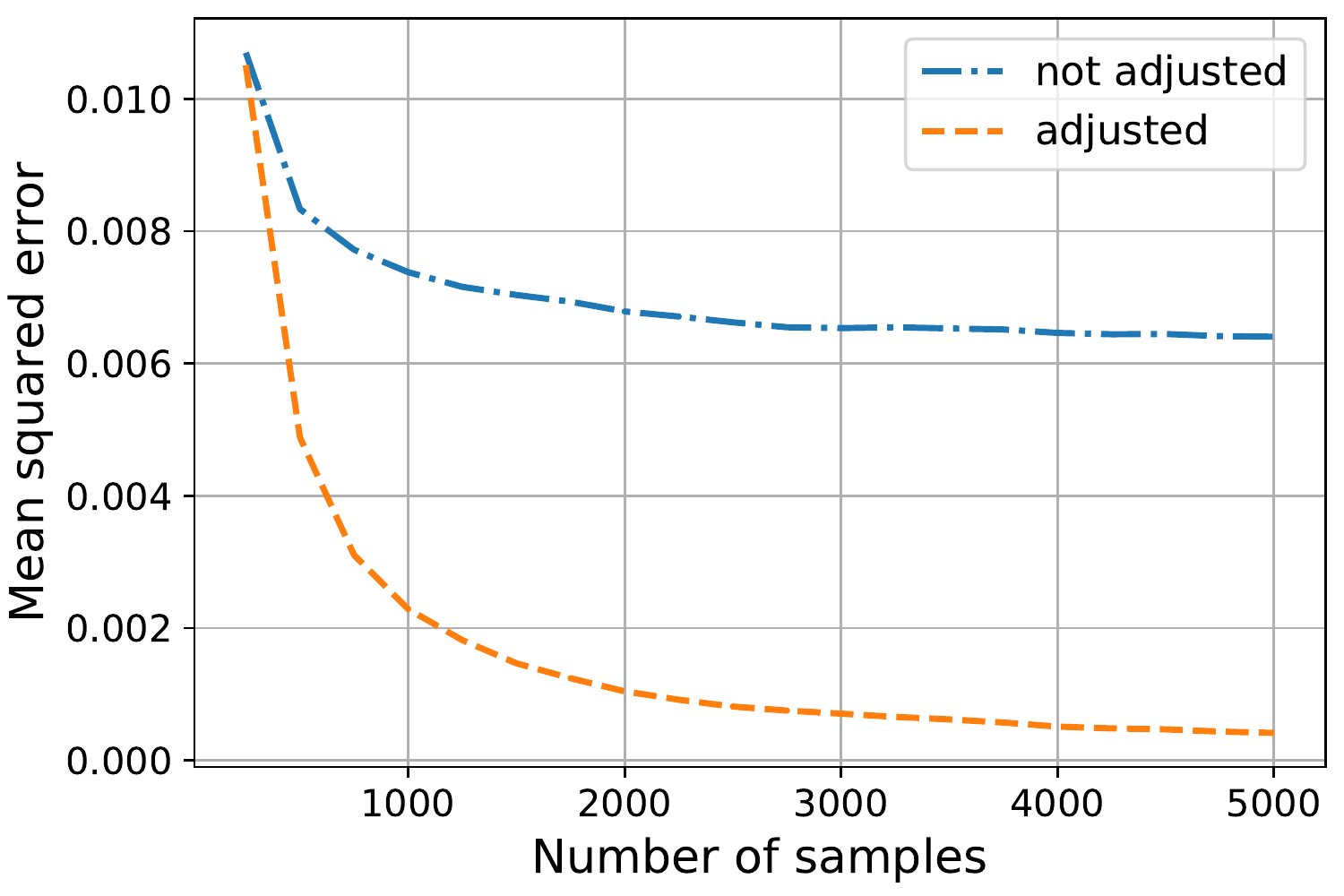}
    }
    \subfigure[]{
		\includegraphics[width=0.31\textwidth]{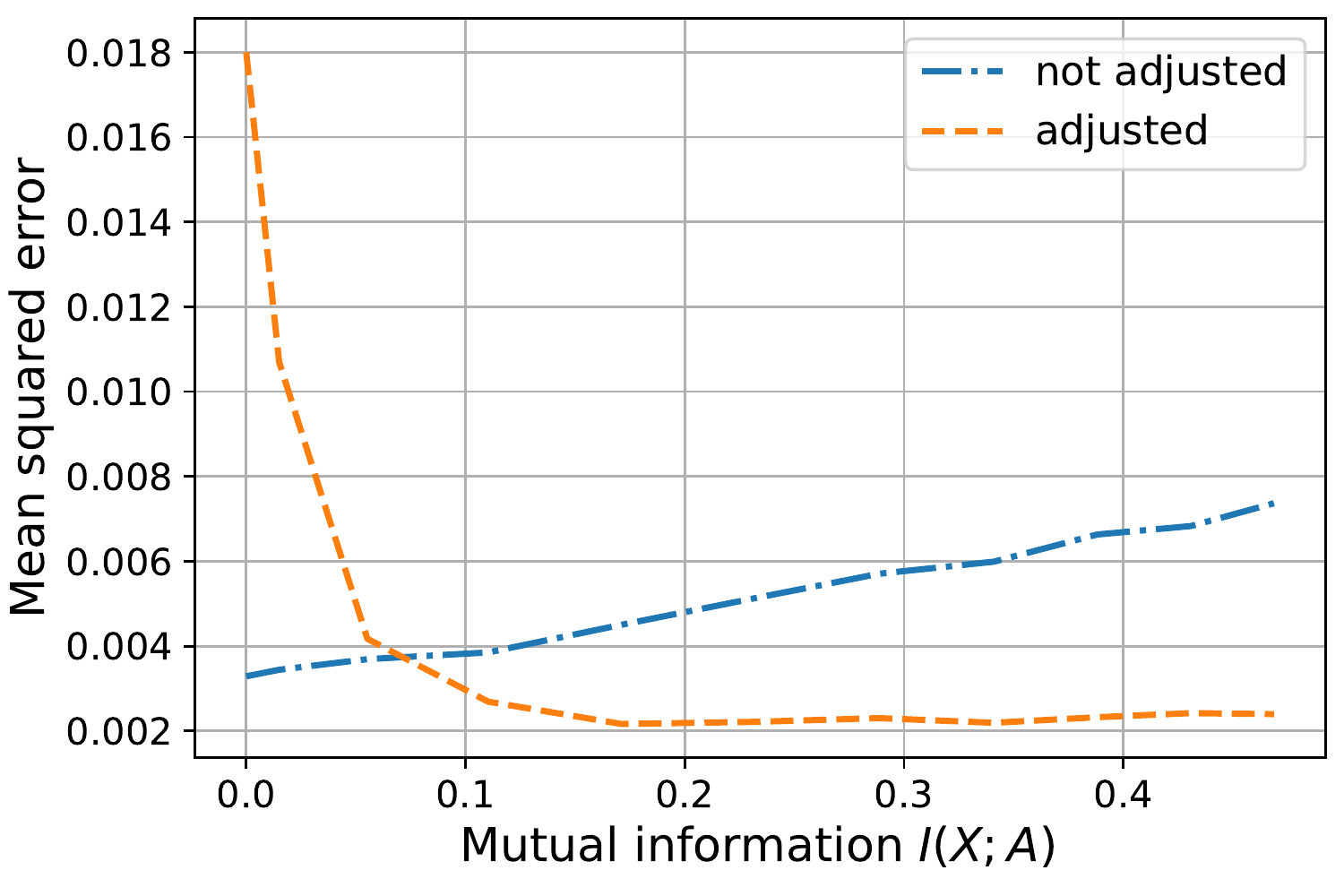}
	}
    \caption{Experimental results on synthetic data. Data was generated from the model described in Equation \ref{eq:synth} with specific parameters varied in each experiment. All plots compare the MSE of the proposed method to a model fit under the assumption that there is no measurement error. In order, the plots show MSE as a function of (a) the true underreporting rate, (b) the sample size, and (c) the mutual information between $A$ and $X$.}\label{fig:synth}
\end{figure*}

We conducted a series of synthetic experiments to demonstrate the behavior of the full likelihood approach based on Theorem \ref{thm:ident2} under varied data conditions and to test the sensitivity of this approach to near violations of the identifiability conditions. Unless otherwise stated, synthetic data was generated using the following model:
\begin{align}
	\nonumber
	X_i &\sim \mathcal{N}(\mathbf{0},\mathbf{I})\\
	\nonumber
	A_i | X_i &\sim \text{Bern}(\text{expit}(\phi_0 + \phi X_i))\\
	\label{eq:synth}
	Z_i &\sim \text{Bern}(1-\tau)\\
	\nonumber
	\tilde{A}_i &= Z_i A_i\\
	\nonumber
	Y_i | A_i, X_i &\sim \text{Bern}(\text{expit}(\theta_0 + \theta_X X_i + \theta_A A_i))
\end{align}
where $\phi$ and $\theta$ were both sampled from a standard normal and $\theta_A$ was set to $1.0$. For any $\theta_X \neq \mathbf{0}$, this data generating process satisfies the identifiability condition of Corollary \ref{cor:cor1}, so the parameters should be identifiable using only a single error-prone observation of the exposure. We evaluated a version of the full likelihood method based on Theorem \ref{thm:ident2} where both $p_\phi(A|X)$ and $p_\theta(Y|A,X)$ were logistic regression models and we maximized the conditional likelihood in Equation \ref{eq:lml} using L-BFGS ("adjusted"). We compared against a logistic regression model estimated under the assumption of no observation error, that is $\tilde{A} = A$ ("not adjusted"). 

In all experiments the target estimand was the risk difference (RD)
%
%
which we estimated as\footnote{Under the assumptions of consistency, positivity, and conditional exchangeability, the risk difference can be interpreted as a measure of the causal effect of $A$ on $Y$ \cite{hernan2018causal}.}:
\begin{align}
	\label{eq:rd_est}
	RD \approx \frac{1}{N}\sum_{i=1}^{N}\,\, &p_{\hat{\theta}}(Y=1|A=1,X=x_i)\\ 
	\nonumber
	&- p_{\hat{\theta}}(Y=1|A=0,X=x_i)
\end{align}
where $\hat{\theta}$ is the estimated value of $\theta$.

While adjusting for measurement error may reduce estimator bias, it may also increase variance\footnote{Adjusting for measurement error in the observations reduces the effective sample size.}. In this way, adjusting for noise in the observation process can be viewed a bias/variance trade-off. Accordingly, we compare the two approaches in terms of mean squared error (MSE) which allows us to judge whether adjusting for measurement errors makes a favorable bias/variance trade-off. Our expectation is that adjusting for measurement error bias leads to reduced estimation error in all cases where the modeling assumptions are satisfied. In all cases, MSE was approximated by averaging over 1,000 synthetic datasets.

In our first experiment (Figure \ref{fig:synth} (a)), we evaluated both approaches under different true underreporting rates by varying $\tau$ between $0$ and $0.8$ while keeping the dataset size fixed at $1,000$. In all cases, adjusting for measurement error reduces MSE, with the gap in MSE increasing as the true underreporting rate increases. This reflects the intuitive idea that less measurement error results in less bias and therefore, there is less to be gained by adjusting for this bias. 

In our second experiment (Figure \ref{fig:synth} (b)), we varied the dataset size between $250$ and $5,000$ while keeping $\tau$ fixed at $0.25$. For small datasets, where we expect variance to swamp bias, both approaches result in comparable MSE; however, as the data size increases, the adjusted estimate converges to zero MSE while the unadjusted estimate converges to a fixed bias and the gap in MSE increases. This result demonstrates the added importance of adjusting for bias when using large datasets so that we do not become overly confident in the wrong inference.

In our third experiment (Figure \ref{fig:synth} (c)), we tested the sensitivity to the assumption in Corollary \ref{cor:cor1} by evaluating how the MSE of the adjusted model increases as the dependence between $A$ and $X$ decreases. Figure \ref{fig:synth} (c) shows this dependence in terms of mutual information. In this experiment, we left $\tau$ fixed at $0.25$ and the data size fixed at $1,000$ while varying the magnitude of $\phi_X$. Specifically, we multiplied the weights $\phi_X$ by a scaler $c$ ranging from $0$ to $1$ ($\phi_0$ was left fixed so that $\mathbf{E}[A]$ remained constant). As we would expect based on Corollary \ref{cor:cor1}, the MSE of the adjusted model increases significantly as the mutual information between $X$ and $A$ approaches zero, with the unadjusted model performing better for some mutual information values greater than zero. An important implication of this result is that if we see a large difference in the estimated variance for the full likelihood and unadjusted approaches, this may be an indication of a near violation of the identifiability conditions. In the next section, we apply the full likelihood method to an effect estimation problem from public health.



%% file: bbc.tex
\begin{figure*}[t!]
    \centering
    \subfigure[]{
		\includegraphics[width=\columnwidth]{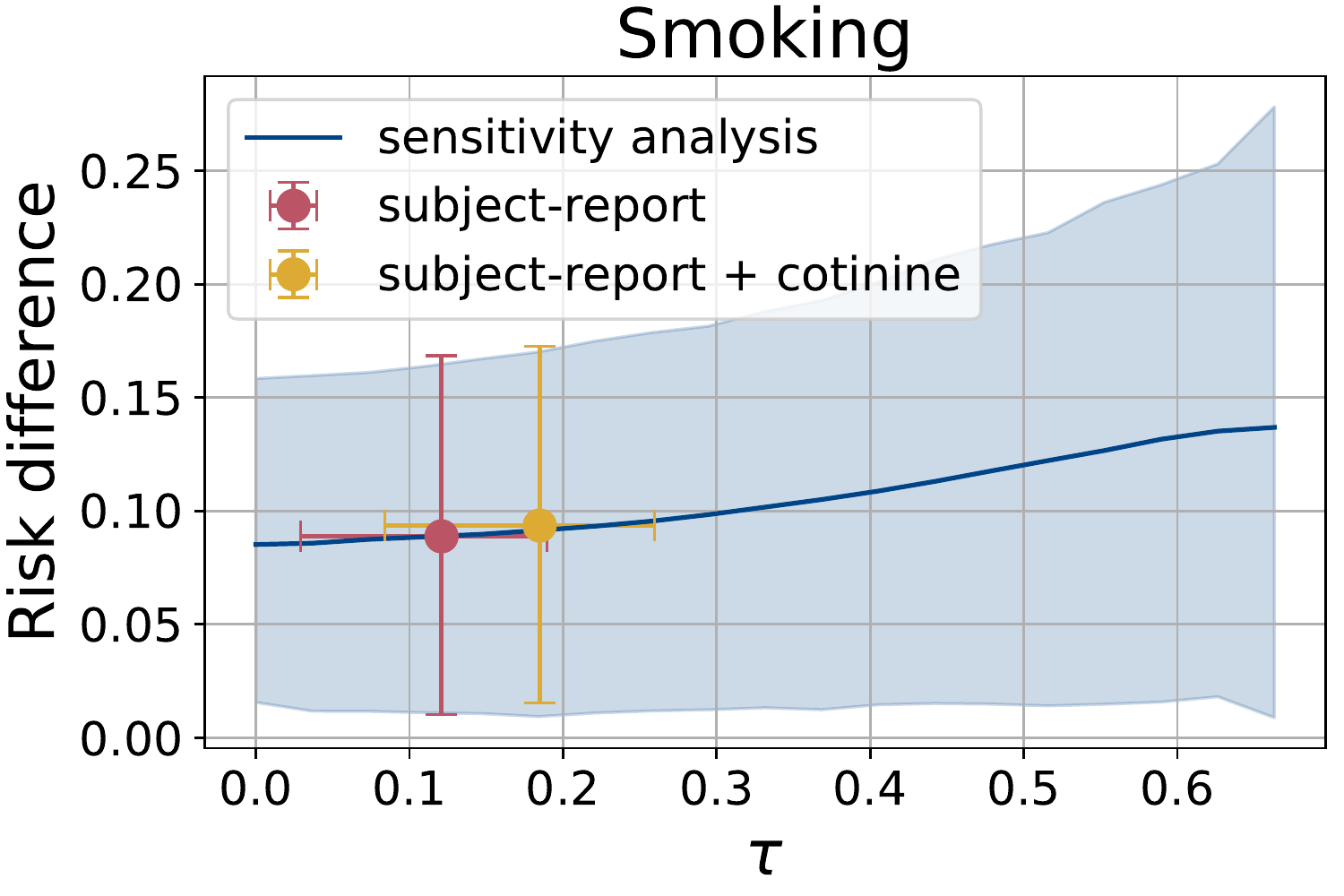}
	}
    \subfigure[]{
		\includegraphics[width=\columnwidth]{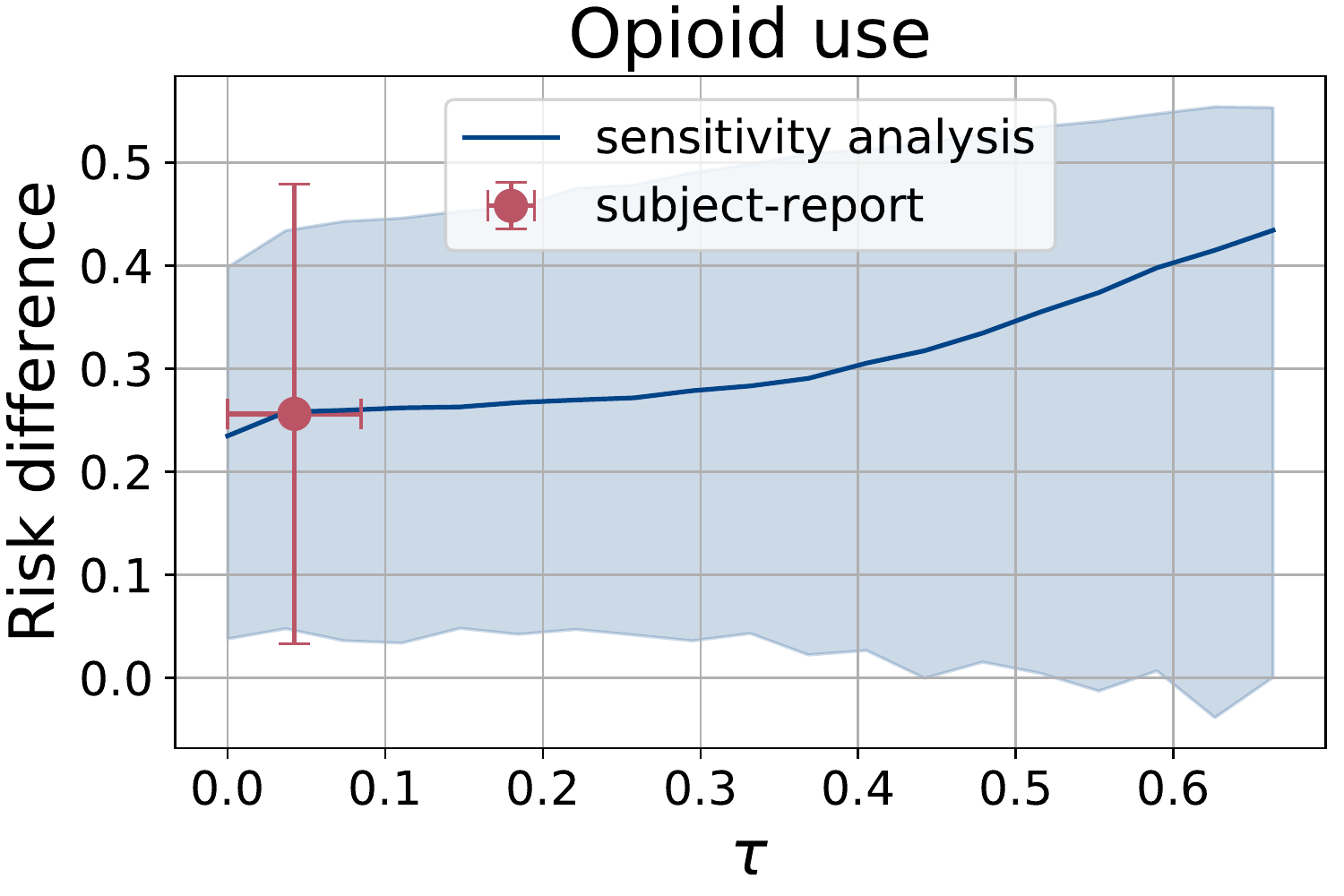}
    }
    \caption{Risk difference estimates for (a) smoking and (b) opioid use. The blue lines show the range of risk difference estimates generated using a sensitivity analysis approach with the shaded region indicating 95\% confidence intervals. Dots represent risk difference estimated generated using a full likelihood method with y-axis error bars indicating 95\% confidence intervals for the risk difference estimate and x-axis error bars indicating 95\% confidence intervals for the estimate of $\tau$.}\label{fig:bbc}
\end{figure*}

Childhood overweight and obesity (COWO) affects over 30\% of children in the United States \cite{ogden2014prevalence} and is associated with a variety of chronic adulthood diseases such as adult obesity \cite{suchindran2010association} and stroke \cite{lawlor2005association}. Understanding the early childhood and pre-birth exposures that lead to COWO is a critical step in developing interventions that may reduce COWO rates. One such potential exposure is maternal drug use during pregnancy, which has been linked to a number of adverse birth outcomes \cite{wang2002maternal,robison2012maternal,whiteman2014maternal}. 

In this section, we use the full likelihood method based on Theorem \ref{thm:ident2} to estimate the effects of smoking and opioid use on COWO. It is difficult to accurately measure drug use in large populations, so we are often forced to rely on subject-reported drug use, which is known to be unreliable \cite{patrick1994validity,gorber2009accuracy,boyd1998quality}. Both the effect of smoking on childhood obesity and the rate at which study subjects underreport smoking have previously been studied (e.g \citet{von2007parental,gorber2009accuracy}) which allows us check our estimates of these quantities agains existing literature. The effect of opioid use on childhood obesity has, to our knowledge, never been studied and we include it as a demonstration of our method on an important unanswered question.


\paragraph{Data:}
To estimate these effects, we use data from the Boston Birth Cohort, a longitudinal dataset tracking various health markers from mothers and children. For each mother/child pair $i$, the outcome $Y_i$ is equal to one if the most recent measurement of the child's body mass index (BMI) is above the 85th percentile, and the true exposure $A_i$ represents whether or not the mother used the substance in question during pregnancy. The complete dataset contains 8,507 mother/child pairs and we have child BMI measurements for 2,763 of those pairs. In the cases of both smoking and opioid use, we have access to subject-reported substance use indicators which we use as error-prone exposure measurements $\tilde{A}_i$. In the case of smoking, we also have measurements of blood cotinine levels (a nicotine metabolite) at the time of delivery for 1,333 of the mothers. As a second error-prone indicator of smoking during pregnancy, we used a binary variable equal to one if this cotinine measurement was available and above the 95th percentile. In both cases, we adjusted for the covariates $X_i$ given in \citet{wang2016weight}, which include (among others) age, income, and other substance use indicators.

\paragraph{Model:}
In all cases, we used logistic regression models for both $p_\phi(a|x)$ and $p_\theta(y|a,x)$. Using a logistic regression model, we estimated the mutual information between $X$ and $\tilde{A}$ at $0.17$ for smoking and $0.14$ for opioid use, which gives us some evidence that $A \not\perp X$, as required by Corollary \ref{cor:cor1}. We estimated $\tau$, $\phi$, and $\theta$ by maximizing the log conditional likelihood in Equation \ref{eq:lml} using L-BFGS. As in our synthetic experiments (Section \ref{sec:synth}) our target estimand was the risk difference, which we estimated according to Equation \ref{eq:rd_est}.

\paragraph{The effect of smoking on childhood obesity:}
When estimating the effect of smoking on childhood obesity, we compared three methods: a sensitivity analysis where $\tau$ was fixed at values ranging from $0$ and $0.65$, the full likelihood method using both subject-reported smoking and cotinine levels, and the full likelihood method using only subject-reported smoking. The resulting estimates along with $95\%$ confidence intervals are shown in Figure \ref{fig:bbc} (a). The sensitivity analysis approach gave a range of estimated risk differences between $0.085$ and $0.137$ while the two full likelihood methods gave similar estimates of $0.089$ and $0.094$. Additionally, the two joint estimation approaches gave estimates for the underreporting rate of $0.121$ and $0.185$.

Three previous studies on the effect of maternal smoking on childhood obesity reported adjusted odds ratios of $1.60$ \cite{von2007parental}, $1.43$ \cite{von2002maternal}, and $1.58$ \cite{toschke2002childhood}. The two full likelihood methods estimated the adjusted odds ratio at $1.51$ and $1.54$, fully in line with the previous literature. Further, a review of previous studies on underreporting rates for smoking found the majority of studies reported underreporting rates between $0.01$ and $0.47$ which is in line with our estimated underreporting rates of $0.121$ and $0.185$ \cite{gorber2009accuracy}.

\paragraph{The effect of opioid use on childhood obesity:}
When estimating the effect of maternal opioid use on childhood obesity, we compared a sensitivity analysis approach to the full likelihood method using subject-reported opioid use as an error-prone exposure measurement. The resulting estimates are shown in Figure \ref{fig:bbc} (b). The sensitivity analysis approach gave a range of estimated risk differences between $0.235$ and $0.434$ while the joint estimation approach estimated the risk difference at $0.256$ and the underreporting rate at $0.042$. 
Using the full likelihood method gives us evidence that the underreporting rate for opioid use is actually quite low and, accordingly, the effect of maternal opioid use on childhood obesity is on the lower end of the range given by the sensitivity analysis approach.
Additionally, at the upper end of the sensitivity analysis range, the $95\%$ confidence intervals grow to include zero. The full likelihood method, on the other hand, gives a confidence interval that excludes zero, giving us stronger evidence that an effect does, in fact, exist.

%% file: related.tex
In this section, we highlight two relevant lines of work. In machine learning, the full likelihood approach has been used to learn classifiers from data in which the target label is subject to measurement error. One of the common approaches to solving this problem is to treat the true label as a latent variable which is marginalized out of the model \cite{yan2010modeling, raykar2009supervised, jin2002learning, mnih2012learning}. Unlike the exposure misclassification problem, the goal in the label noise problem is generally prediction, so models are evaluated empirically in terms of prediction accuracy with no guarantees about model identifiability.

The full likelihood approach has also been applied to a problem in computational ecology called \emph{occupancy modeling} that shares important properties with the exposure underreporting problem. In the occupancy modeling problem the goal is to estimate the probability that a region is inhabited by a particular species given features of the region. Occupancy models are typically estimated from survey data which often shares the same strict underreporting property that we discussed in Section \ref{sec:model}. That is, if a species is observed, we assume the site is, in fact, occupied by the species. This problem differs from ours in that the 
measurement error is present in the target $Y$ rather than the exposure $A$, but the marginal likelihood approach used \citet{solymos2012conditional} is very similar to ours. In \citet{solymos2012conditional} and the resulting discussion \cite{knape2015estimates,solymos2016revisiting,knape2016assumptions} the authors present identifiability conditions for their model and the results presented in Section \ref{sec:identifiability} are extensions of these conditions to the problem of exposure misclassification.

%% file: discussion.tex

Measurement error is widespread in observational data and can lead to inferential bias with severe real-world implications. We argued that performing quantitative bias analysis and adjustment should be a first order concern for researchers using observational data and we presented a new method to adjust for the bias caused by exposure underreporting. We showed that, under certain assumptions, this method can be used in a variety of scenarios, including when only a single error-prone exposure observation is available, a capability that did not exist before.

We demonstrated on synthetic data that this method can potentially reduce estimation error by a significant amount relative to ignoring measurement error, but may be sensitive to near violations of the modeling assumptions. Finally, we used this method to estimate the effects of maternal smoking and opioid use during pregnancy on childhood obesity. Our method refined the range of estimates given by a sensitivity analysis approach and, in the case of smoking, resulted in estimates matching previous literature. 

There are a two additional points about this method that we would like to highlight. First, we do not view our method as a replacement to traditional sensitivity analysis, but rather as a complementary approach that can be used to improve and refine the range of estimates generated by sensitivity analysis. Specifically, when performing bias analysis, we recommend generating the types of plots shown in Figure \ref{fig:bbc} which combine results from both full likelihood and sensitivity analysis approaches. By analyzing the data under multiple sets of assumptions, we make our results more robust to violations of any single assumption\footnote{This is the very principle underlying sensitivity analysis.}.

Second, while Theorem \ref{thm:ident2} allows us to use several common Bernoulli regression models it is worth considering which models are excluded by this condition. An example of one such model is a scaled logistic regression model of the form $p_\phi(A=1|x) = \alpha \text{expit}(\phi x)$ where $\alpha \in [0,1]$. 
In this model, $p_\phi(A=1|x)$ saturates at a value $\alpha \leq 1$ which allows us to control the maximum value of $p_\phi(A=1|x)$ independently from sharpness of the decision boundary. As suggested by \citet{solymos2016revisiting}, such a model may be approximated by using logistic regression paired with a polynomial basis expansion of the covariates. 
While this type of feature expansion does not violate the condition in Corollary \ref{cor:cor1}, it may increase estimator variance to the point where regularization becomes necessary to stabilize estimation. For a discussion this stabilization and further models that violate this condition, see \citet{knape2015estimates}, \citet{solymos2016revisiting}, and \citet{knape2016assumptions}. 


There are several potential future directions for this work. First, relaxing the conditional independence and strict underreporting assumptions from Section \ref{sec:model} would allow broader application of this method. For example, \citet{solymos2012conditional} suggest an alternative to Corollary \ref{cor:cor1} for the occupancy modeling problem that can be used when $\tilde{A} \not\perp X | A$. Similarly, while we think it unlikely that we can eliminate all assumptions about the error distribution, generalizing the strict underreporting assumption to allow for false positives is necessary for many potential applications. Second, \citet{rothman2008modern} suggest using semi-bayesian methods instead of sensitivity analysis to generate a distribution over inferences. In a similar vein, a fully Bayesian version of the approach presented in this paper would allow the modeler to refine the sensitivity analysis approach in a principled way while potentially stabilizing the full likelihood approach in the case of near violations of the identifiability conditions.


%% file: appendix.tex
\section{Proof of Corollary \ref{cor:cor1} for probit and cloglog models}
In this section we prove Corollary \ref{cor:cor1} in Section \ref{sec:single} of the main paper for the cases of probit and cloglog regression.

\begin{proof}
Following the proof for logistic regression, a probit regression model violates the condition in Theorem \ref{thm:ident2} if

\begin{align*}
	\sqrt{2}\,\text{erf}^{-1}\left(\alpha - 1 + \alpha\,\text{erf}\left(\frac{\phi x}{\sqrt{2}}\right)\right) = \phi' x
\end{align*}

Similarly, a cloglog regression model violates the condition in Theorem 1 if 

\begin{align*}
	\log(-\log(1-\alpha+\alpha\exp(-\exp(\phi x)))) = \phi' x
\end{align*}

In both of these cases, the function on the left hand side is non-linear for $\alpha < 1$, so these equalities can only be true if $\phi x$ and $\phi' x$ are constants which is true only when $A \perp X$.
\end{proof}